\documentclass[conference]{IEEEtran}
\usepackage{cite}
\usepackage{amsmath,amssymb,amsfonts,amsthm}
\usepackage{algorithm}
\usepackage{algorithmic}
\usepackage{graphicx}
\usepackage{textcomp}
\usepackage{xcolor}
\usepackage{hyperref}
\usepackage{threeparttable}
\usepackage{booktabs}
\usepackage{xcolor,pifont}
\usepackage[table]{xcolor}

\newcommand{\cmark}{\textcolor{green}{\ding{51}}}%
\newcommand{\xmark}{\textcolor{red}{\ding{55}}}%
\newcommand{\textbbf}[1]{{\bfseries\textcolor{blue}{#1}}}

\theoremstyle{plain} 
\newtheorem{theorem}{Theorem}
\newtheorem{lemma}{Lemma}
\newtheorem{corollary}{Corollary}

\theoremstyle{definition}
\newtheorem{definition}{Definition}
\newtheorem{assumption}{Assumption}

\def\BibTeX{{\rm B\kern-.05em{\sc i\kern-.025em b}\kern-.08em
    T\kern-.1667em\lower.7ex\hbox{E}\kern-.125emX}}
\begin{document}

\title{LexiSafe: Offline Safe Reinforcement Learning with Lexicographic Safety-Reward Hierarchy\\
\thanks{Identify applicable funding agency here. If none, delete this.}
}

\author{\IEEEauthorblockN{Hsin-Jung Yang}
\IEEEauthorblockA{\textit{Dept. of Mechanical Engineering} \\
\textit{Iowa State University}\\
Ames, USA \\
hjy@iastate.edu}
\and
\IEEEauthorblockN{Zhanhong Jiang}
\IEEEauthorblockA{\textit{\quad Translational AI Center\quad} \\
\textit{Iowa State University}\\
Ames, USA \\
zhjiang@iastate.edu}
\and
\IEEEauthorblockN{Prajwal Koirala}
\IEEEauthorblockA{\textit{Sibley School of Mechanical and Aerospace Engineering} \\
\textit{Cornell University}\\
Ithaca, USA \\
pk596@cornell.edu}
\and
\IEEEauthorblockN{Qisai Liu}
\IEEEauthorblockA{\textit{Dept. of Mechanical Engineering} \\
\textit{Iowa State University}\\
Ames, USA \\
supersai@iastate.edu}
\and
\IEEEauthorblockN{Cody Fleming}
\IEEEauthorblockA{\textit{Dept. of Mechanical Engineering} \\
\textit{Iowa State University}\\
Ames, USA \\
flemingc@iastate.edu}
\and
\IEEEauthorblockN{Soumik Sarkar}
\IEEEauthorblockA{\textit{Dept. of Mechanical Engineering} \\
\textit{Iowa State University}\\
Ames, USA \\
soumiks@iastate.edu}
}

\maketitle

\begin{abstract}
Offline safe reinforcement learning (RL) is increasingly important for cyber-physical systems (CPS), where safety violations during training are unacceptable and only pre-collected data are available. Existing offline safe RL methods typically balance reward–safety tradeoffs through constraint relaxation or joint optimization, but they often lack structural mechanisms to prevent safety drift. We propose LexiSafe, a lexicographic offline RL framework designed to preserve safety-aligned behavior. We first develop LexiSafe-SC, a single-cost formulation for standard offline safe RL, and derive safety-violation and performance-suboptimality bounds that together yield sample-complexity guarantees. We then extend the framework to hierarchical safety requirements with LexiSafe-MC, which supports multiple safety costs and admits its own sample-complexity analysis. Empirically, LexiSafe demonstrates reduced safety violations and improved task performance compared to constrained offline baselines. By unifying lexicographic prioritization with structural bias, LexiSafe offers a practical and theoretically grounded approach for safety-critical CPS decision-making. 

\end{abstract}


\section{Introduction}
Reinforcement learning (RL) has achieved remarkable success across diverse domains such as robotics~\cite{garaffa2021reinforcement}, manufacturing~\cite{samsonov2022reinforcement}, recommender systems~\cite{afsar2022reinforcement}, healthcare~\cite{yu2021reinforcement}, and even reasoning with large language models~\cite{zhai2024fine}. However, when applied to cyber-physical systems (CPS), such as autonomous driving~\cite{kiran2021deep}, smart grids~\cite{lu2018dynamic}, building energy management~\cite{yu2021review}, conventional RL faces critical limitations. These systems tightly couple computation and physical processes, where unsafe actions can directly cause physical harm, equipment failure, or service disruption. Ensuring safety is therefore not only desirable but mandatory for real-world deployment. This requirement is further amplified by the inherent vulnerabilities of deep RL agents, which often lack natural robustness to environmental perturbations \cite{LIU2024126} and remain susceptible to adversarial threats \cite{Lee_Ghadai_Tan_Hegde_Sarkar_2020}.

\begin{figure}[t]
    \centering
    \includegraphics[width=0.95\linewidth]{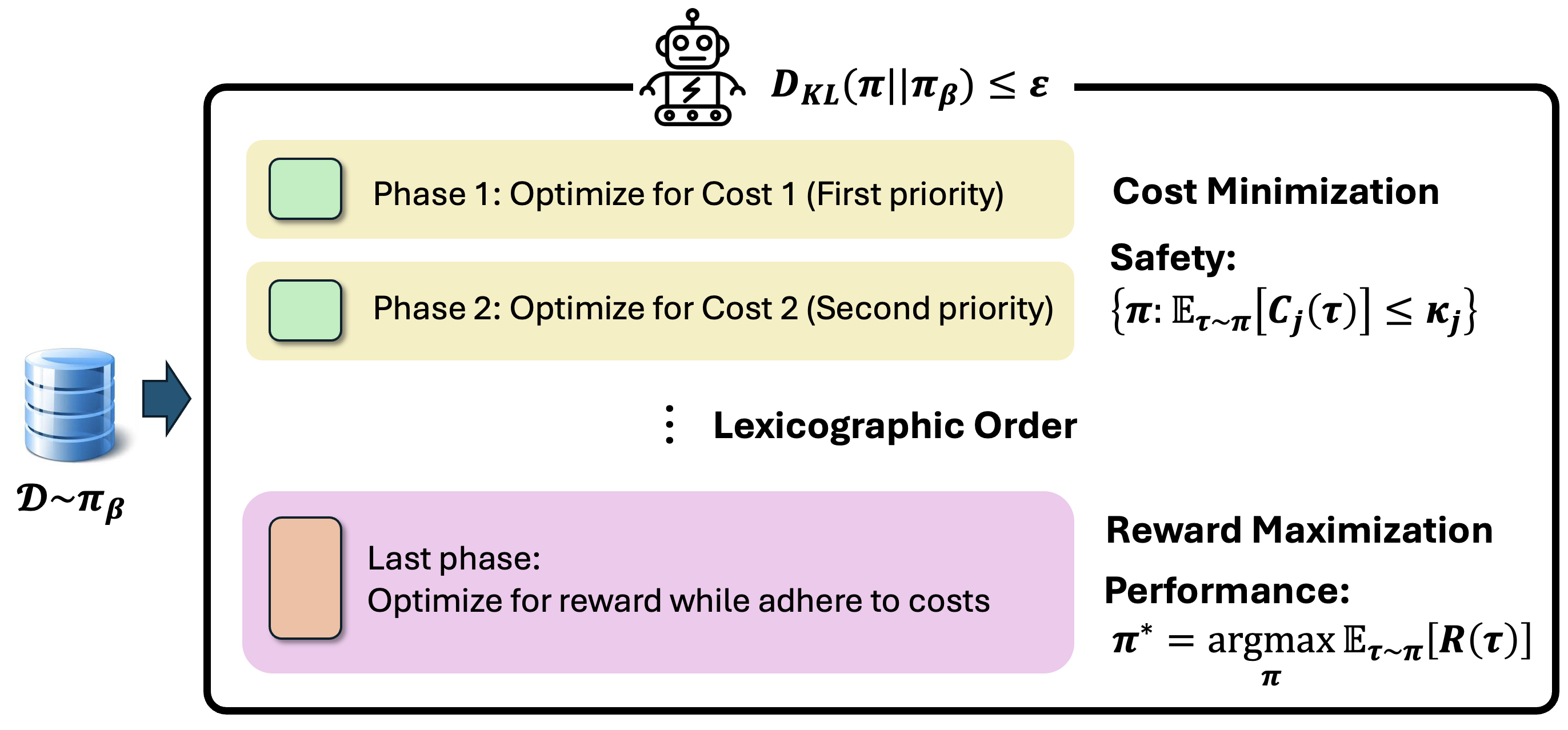}
    \caption{\textbf{LexiSafe}: The agent learns from an offline dataset $\mathcal{D}\sim\pi_\beta$ under a distributional shift constraint $D_{KL}(\pi||\pi_\beta)\leq \varepsilon$. In Stage 1 (phases marked in yellow boxes), the actor network is trained to minimize cumulative costs under constraints with safety hierarchy. In Stage 2 (phase in pink), the model is retrained to maximize reward. This enforces a lexicographic policy update, preserving safety while optimizing performance. Please see Definition~\ref{definition_3} for the formula in the Figure.}
    \label{fig:laxisafe_diagram}
\end{figure}

In CPS applications, safety often involves multiple and hierarchical constraints, rather than a single-cost signal. For instance, in autonomous driving, an agent must first avoid collisions (primary safety), then respect traffic regulation (secondary safety), and finally optimize fuel efficiency or passenger comfort (performance). Violating this hierarchy, e.g., prioritizing comfort over collision avoidance, is unacceptable. This multi-level safety dependency motivates a \textit{lexicographic} structure, where safety objectives are optimized sequentially according to their criticality before considering performance. Nevertheless, existing safe RL approaches rarely capture such safety hierarchies, treating safety and performance as jointly optimized under a single constraint.

In practice, direct online interaction for learning safe behavior in CPS is costly and risky, as unsafe exploration can lead to physical damage or system instability. This has motivated the study of offline safe RL~\cite{xu2022constraints,chemingui2025constraint}, where policies are trained from pre-collected datasets without further environment interactions. However, this setting introduces several challenges. Offline datasets often contain mixed or unsafe trajectories, complicating the identification of safe behaviors~\cite{koirala2024latent}. Furthermore, estimation errors in long-term cost and value functions may yield infeasible or overly conservative policies. While dual-variable or constrained formulations~\cite{koirala2024fawac,zheng2024safe,hong2024primal} attempt to balance safety and performance, they often suffer from optimization instability and lack of interpretability. Even if theoretical complexity bounds have been estabilished separately from safe RL~\cite{hasanzadezonuzy2021learning} and offline RL~\cite{hong2023offline}, analogous guarantees for offline safe RL, particularly under hierarchical safety objectives, remain underexplored.

Therefore, limitations motivate the central question of this work:
    \textit{How can we ensure hierarchical safety guarantees in offline reinforcement learning for cyber-physical systems, while still achieving near-optimal task performance?}
Recently, a few studies~\cite{skalse2022lexicographic,zhang2022lexicographic} have explored lexicographic to model hierarchical objectives. However, existing methods primarily focus on online settings, where safety and performance are optimized through continual environment exploration. Such approaches lack theoretical sample complexity guarantees and are typically limited to single-cost evaluations, making them difficult to deploy in safety-critical CPS domains that demand strict offline learning and multiple safety hierarchies.

\noindent\textbf{Contributions.}
To address these gaps, in this paper, we leverage \textit{lexicographic order}, which is of independent interest in recent
multi-objective RL literature~\cite{skalse2022lexicographic,xue2025multi}. We introduce \textit{LexiSafe} as in Figure~\ref{fig:laxisafe_diagram}, which addresses the fundamental tension between safety and performance in offline RL by introducing a lexicographic framework with multi-phase training. Unlike prior methods that relax constraints or sequentially train separate safety/performance models, LexiSafe unifies safety and performance by treating safety as a non-negotiable priority (one or multiple lexicographic safety objectives) and performance as a secondary goal, ensuring policy adheres to safety boundaries. Particularly, the multi-phase optimization enables the different cost minimization, ensuring the hierarchical safety priorities before the reward maximization. Our method theoretically grounds this mechanism with the first sample complexity bounds for lexicographic safe RL. LexiSafe demonstrates empirical dominance on the DSRL benchmark, outperforming constrained baselines across robotic manipulation and autonomous driving tasks by strictly enforcing safety and accelerating convergence. 
\noindent The main contributions are summarized in the following: 
(1) We propose LexiSafe (both \textit{LexiSafe-SC} and \textit{LexiSafe-MC}, SC and MC indicate single-cost and multi-cost), a novel framework that hierarchically separates safety constraints from performance optimization, ensuring safety violations are eliminated after initial convergence;
(2) We formally establish the constraint violation and performance suboptimality bounds for the single-cost scenario. On top of this, we derive the first sample complexity bounds for both single- and multi-cost cases, linking guarantees to policy architecture; 
(3) We validate our proposed LexiSafe by comparing it to multiple baselines on different DSRL benchmarks. With both single- and multi-cost scenarios, we show that the proposed LexiSafe outperforms baselines and achieves the best safety-compliant performance.
\begin{table}[t]
\caption{Comparison of methods with lexicographic order. $d$: feature dimension, $T$: episode length, $\epsilon$: accuracy.}
\begin{center}
\begin{threeparttable}
\begin{tabular}{c c c}
    \toprule
    \textbf{Method} & \textbf{Offline Setting} & \textbf{Complexity} \\ \midrule
      PBLRL \cite{skalse2022lexicographic}   &  \xmark                  &     \xmark                 \\
      PSQD \cite{rietz2023prioritized}   &  \xmark                   &   \xmark                   \\ 
      LPA \cite{tercan2024thresholded} & \xmark & \xmark \\
      LLRL \cite{xue2025multi} & \xmark & 
      $\tilde{\mathcal{O}}(\frac{d^2T^4}{\epsilon^2})$\\
      \textbf{LexiSafe-SC} & \cmark & Theorem~\ref{theorem_3}\\
      \textbf{LexiSafe-MC} & \cmark & Corollary~\ref{corollary_1}\\
      \bottomrule
\end{tabular}
\end{threeparttable}
\end{center}
\label{table:method_com}
\end{table}

\section{Related Work}
\noindent\textbf{Offline safe RL.}
Offline safe RL synthesizes offline RL~\cite{fu2022closer} and safe RL~\cite{gu2024review} intelligently such that an optimal safe policy is learned exclusively from an offline dataset. Despite some contemporary attempts to solve this class of problems, most of them are inadequate for handling constrained optimization and conservative learning concurrently. In one of early works, a batch policy learning algorithm~\cite{le2019batch} was developed with constraints by directly resorting to off-policy safety evaluation, while confining the approach to discrete action space. Pertaining to this method, a projection algorithm was proposed in ~\cite{polosky2022constrained} to project performance maximization policy back to safety-respecting region through the Fenchel duality. To further address conservatism issue in traditional pessimistic methods, Zhang et. al~\cite{zhang2023constrained} utilized a flow-GAN model to explicitly estimate the density of behavior policy, which enables optimization within a safe region. Orthogonal to existing approach, one recent work~\cite{koirala2024latent} first inferred latent safety constraints and then maximized reward by complying with such constraints. However, they require to train separate encoders for minimizing cost and maximizing reward.
More recently, diffusion models are also used for guiding the policy learning~\cite{zheng2024safe,mao2024diffusion}, trajectory generation~\cite{romer2024safe,zhang2025enhancing}, and behavior regularization~\cite{guo2025reward}. Compared with the prior work, our lexicographic framework ensures safety adherence throughout optimization within a single model.

\noindent\textbf{Lexicographic RL.}
Lexicographic RL (LRL)~\cite{skalse2022lexicographic} has recently become a competitive method as it involves ordering objectives by importance and optimizing them sequentially. 
This means the agent first focuses on satisfying the highest-priority objective, and only then considers lower-priority objectives, essentially treating them as constraints. This approach contrasts with methods that try to optimize all objectives simultaneously, which can be difficult to balance, particularly when objectives conflict. Skalse et al.~\cite{skalse2022lexicographic} proposed value-based and policy-based LRL methods and compared their approaches to baselines in solving constrained RL problems. Though in their work safe RL was considered as a natural application in LRL, offline safe RL has not been discussed with any detail.
Another work developed lexicographic actor-critic algorithm~\cite{zhang2022lexicographic} for urban autonomous driving, demonstrating empirical efficacy of LRL in hierarchical reward settings. To address continuous space lexicographic multi-objective RL problems, the authors in~\cite{rietz2023prioritized} proposed prioritized soft Q-decomposition for learning and adapting subtask solutions under lexicographic priorities. A more recent work~\cite{tercan2024thresholded} showed the shortcomings of thresholded lexicographic Q-learning and developed lexicographic projection algorithm to address the problem, by computing a lexicographically optimal direction to optimize the present unsatisfied highest importance objective while preserving the values of more important objectives using
projections onto hypercones of their gradients. To make LRL in linear MDPs theoretically grounded, Xue et al.~\cite{xue2025multi} established the sample complexity bound for the first time. While LRL provides a formal framework for constraint satisfaction under strict priority ordering, its application to offline safe RL remains nascent and severely understudied.

\section{Preliminaries and Problem Formulation}
\noindent\textbf{Offline safe RL.}
We consider a Markov decision process (CMDP) represented by the tuple $\mathcal{M}=(\mathcal{S}, \mathcal{A}, \mathcal{P}, r, c, \gamma, d_0)$. Herein, $\mathcal{S}$ signifies the state space, $\mathcal{A}$ the action space, $\mathcal{P}:\mathcal{S}\times\mathcal{A}\times\mathcal{S}\to[0,1]$ the state transition dynamics, $r:\mathcal{S}\times\mathcal{A}\to[0,r_m]$ the immediate reward function, bounded by some constant $r_m>0$, $c:\mathcal{S}\times\mathcal{A}\to[0,c_m]$ the cost function upper bounded by some constant $c_m>0$, $\gamma$ the discount factor, and $d_0$ the initial state distribution. 
Intuitively, the cost function $c$ sanctions safety-critical violations, compelling the agent to restrict cumulative penalties to a predetermined safety budget $\kappa$. A stochastic policy represented by $\pi (a|s):\mathcal{S}\to\mathcal{A}$, defines a mapping from the state $s$ to a probability distribution over actions $a$. We then also define the stationary state-action distribution under the policy $\pi$ as $d^\pi(s,a)=(1-\gamma)\sum_{h=0}^H\gamma^hp(s_h=s,a_h=a)$, where $p$ indicates the probability.
We denote a trajectory by $\tau= \{(s_0,a_0,r_0,c_0),(s_0,a_0,r_0,c_0),...,(s_T,a_T,r_T,c_T)\}$ that consists of a sequence of states, actions, rewards and costs over time. $T$ is the length of trajectory. Likewise, the discounted cumulative reward for a trajectory $\tau$ is defined as $R(\tau)=\sum_{t=0}^T\gamma^tr(s_t,a_t)$, and the discounted cumulative cost is $C(\tau)=\sum_{t=0}^T\gamma^tc(s_t,a_t)$. In our work, we learn a policy in safe RL from an offline dataset. Denote by $\mathcal{D}:=(\{s,a,s',r,c\})$, possibly with both safe and unsafe trajectories. We also define the value functions for both reward and cost in a unified way as $V^\pi_\nu(s)=\mathbb{E}_{\tau\sim\pi}[\sum_{t=0}^T\gamma^t\nu_t|s_t=s], \nu\in\{r,c\}$. The Q-functions for both are defined as $Q^\pi_\nu(s,a)=\mathbb{E}_{\tau\sim\pi}[\sum_{t=0}^T\gamma^t\nu_t|s_t=s,a_t=a]$.
We also denote by $\pi_\beta$ the unknown data-collecting policy for $\mathcal{D}$. Thus, the offline safe RL is formulated in the following:
\begin{equation}\label{eq_2}
    \text{max}_\pi V^\pi_r(s), \; \text{s.t.,} V^\pi_c(s)\leq \kappa;\; D_{KL}(\pi||\pi_\beta)\leq \varepsilon,
\end{equation}
where $D_{KL}(\cdot||\cdot)$ indicates the Kullback-Liebler (KL) divergence and $\varepsilon$ is a tolerance parameter. The constraints in Eq.~\ref{eq_2} ensures that the learned policy $\pi$ not only remains within the safe region defined by the cost threshold $\kappa$ but also close to the behavior policy $\pi_\beta$, mitigating the risk of producing out-of-distribution (OOD) actions.

\noindent\textbf{Lexicographic order.} From Eq.~\ref{eq_2}, we have known that offline safe RL is formulated as an optimization problem of maximizing one reward subject to two constraints. This problem setup to some extent resembles a multi-objective context where the policy should be learned first to encode the safety constraint from the offline dataset, and then to maximize the performance while complying with safety. The learning process could cause conflicts between safety and performance as some policies out of the safe region may achieve better performance. As conflicting objectives cannot be maximized concurrently, a notion of \textit{order} is defined in multi-objective optimization for prioritizing one objective (or constraint) over another~\cite{pineda2015revisiting,xue2025multi}. Here, a connection between CMDP and lexicographic MDP (LMDP) was already established: solving a CMDP is equivalent to solving a sequence of $k$ LMDPs. Corresponding to Eq.~\ref{eq_2}, $k=2$ such that we define formally the lexicographic order in LexiSafe:
\begin{definition}\label{definition_1}
    Given a pre-collected dataset $\mathcal{D}$ from some unknown behavior policy $\pi_\beta$, LexiSafe-SC enforeces a strict hierarchical priority between objectives:
    \begin{itemize}
        \item Primary objective (safety):
        \[\pi_\text{safe}\in\text{argmin}_\pi\mathbb{E}_{\tau\sim\pi}[C(\tau)] \;\text{s.t.,}\mathbb{E}_{s\sim\mathcal{D}}[D_{KL}(\pi||\pi_\beta)]\leq\varepsilon.\]
        A policy $\pi$ is feasible if it satisfies $\mathbb{E}_{\tau\sim\pi}[C(\tau)]\leq\kappa$.
        \item Secondary objective (performance): \[\pi^*=\text{argmax}_{\pi\in\Pi}\mathbb{E}_{\tau\sim\pi}[R(\tau)] \;\text{s.t.,}\mathbb{E}_{s\sim\mathcal{D}}[D_{KL}(\pi||\pi_\beta)]\leq\varepsilon,\]
        where $\Pi=\{\pi:\mathbb{E}_{\tau\sim\pi}[C(\tau)]\leq\kappa\}$.
    \end{itemize}
\end{definition}
Intuitively, the lexicographic order turns Eq.~\ref{eq_2} into solving two objectives sequentially with different priorities: \textit{safety first, then performance}. Notably, in both objectives, KL divergence as a proximity constraint enforces the learned policy to remain close to the behavior policy since in our work we will leverage the same dataset in these two different phases. Also, note that this hierarchical learning can be done within each epoch or separately. To establish the sample complexity for LexiSafe, we have to construct the cost constraint violation and performance suboptimality bounds first. Hence, our analytical results are primarily dedicated to the single-cost scenario. Once these results are ready, the complexity bounds for both cost scenarios can be obtained.

\noindent\textbf{Implicit Q-Learning.}
To mitigate distributional shift, offline RL algorithms typically employ regularization to constrain the policy and/or critic. Some approaches avoid updates entirely beyond the dataset support. Implicit Q-learning (IQL) exemplifies this strategy by training a state-value network to prevent OOD action queries in the Q-network~\cite{kostrikov2021offline}. Critically, IQL requires no policy-derived actions as training occurs exclusively on dataset actions. The standard losses are:
\begin{equation}\label{eq_3}
    \mathcal{L}_Q(\rho)=\mathbb{E}_{(s,a,s')\sim\mathcal{D}}[(r+\gamma V_\phi(s')-Q_\rho(s,a))^2]
\end{equation}
\begin{equation}\label{eq_4}
    \mathcal{L}_V(\phi)=\mathbb{E}_{(s,a)\sim\mathcal{D}}[L^2_\xi(Q_\rho(s,a)-V_\phi(s))]
\end{equation}
Eqs.~\ref{eq_3} and~\ref{eq_4} detail the training of value network ($V_\phi$) and Q-network ($Q_\rho$). Particularly, the value network is trained with expectile regression objective and leverages an asymmetric squared error loss function, which is defined as $L^2_\xi(u)=|\xi-1(u<0)|u^2$, where $\xi\in(0.5,1.0)$. Given trained Q and value networks, the policy is optimized via advantage-weighted regression (AWR)~\cite{peng2019advantage}, which first obtains a non-parametric closed-form solution for the policy update, and then projects this solution into the parameterized space of the policy network. This approach leverages the learned advantage function $A(s,a)=Q_\rho(s,a)-V_\phi(s)$ to steer updates while strictly adhering to the dataset's distributional constraints. 
Implicit Q-Learning (IQL) in offline RL learns state-value functions directly from dataset experiences, then extracts high-return policies through advantage-weighted regression. We extend this framework to safety constraints by simultaneously learning a cost-value function $V^c_\eta(s)$ from the same offline data. Defining the cost-advantage as $A^c(s,a)=Q^c_\psi(s,a)-V^c_\eta(s)$, we apply an asymmetric loss during $V^c_\eta(s)$ training to deliberately avoid underestimation of $Q^c_\psi(s,a)$. This dual-value approach maintains IQL's data efficiency while enabling policy optimization that respects safety boundaries through cost-advantage weighting. The specific losses are as follows:
\begin{equation}\label{eq_5}
    \mathcal{L}^c_Q(\psi)=\mathbb{E}_{(s,a,s')\sim\mathcal{D}}[(c+\gamma V_\eta(s')-Q^c_\psi(s,a))^2]
\end{equation}
\begin{equation}\label{eq_6}
    \mathcal{L}^c_V(\eta)=\mathbb{E}_{(s,a)\sim\mathcal{D}}[L^2_\xi(Q^c_\psi(s,a)-V^c_\eta(s))]
\end{equation}

\section{Proposed Method}\label{sec:method}
\begin{algorithm}[t]
\caption{LexiSafe-SC, single-cost single reward}
\label{alg:algorithm}
\begin{algorithmic}[1] 
\STATE \textbf{Initialization}: $\rho,\phi$ for reward Q and V nets, $\psi, \eta$ for cost Q and V nets, policy parameters $\theta$, Lagrangian multiplier $\lambda\geq 0$, learning rate $\nu_c,\nu_r, \nu_\lambda$
\FOR{each gradient step}
\STATE \texttt{Update cost critics:}
\STATE $\psi\leftarrow\psi-\nu_c\nabla_\psi\mathcal{L}^c_Q(\psi)$
\STATE $\eta\leftarrow\eta-\nu_c\nabla_\eta\mathcal{L}^c_V(\eta)$
\STATE \texttt{Update reward critics:}
\STATE $\rho\leftarrow\rho-\nu_r\nabla_\rho\mathcal{L}^r_Q(\rho)$
\STATE $\phi\leftarrow\phi-\nu_r\nabla_\phi\mathcal{L}^r_V(\phi)$
\STATE \texttt{Phase 1: Cost minimization}
\STATE $A^c(s,a)\leftarrow Q^c_\psi(s,a)-V^c_\eta(s)$
\STATE $\theta\leftarrow\theta-\nu_c\nabla_\theta\mathcal{L}^c_\pi(\theta)$
\STATE \texttt{Phase 2: Reward maximization}
\STATE $\lambda\leftarrow\text{max}\{0,\lambda+\nu_\lambda(\hat{C}(\tau)-\kappa)\}$
\STATE $A^r(s,a)\leftarrow Q_\rho(s,a)-V_\phi(s)$
\STATE $\theta\leftarrow\theta-\nu_r\nabla_{\theta}\mathcal{L}^r_\pi(\theta)$
\ENDFOR
\STATE \textbf{Output}: $\pi_\theta$
\end{algorithmic}
\end{algorithm}

\noindent\textbf{Safety learning.} In this phase, as indicated in Definition~\ref{definition_1}, the goal is to minimize the expected discounted cost $\mathbb{E}_{\tau\sim\pi}[C(\tau)]$ subject to the safety constraint $\mathbb{E}_{\tau\sim\pi}[C(\tau)]\leq\kappa$ and proximity constraint $\mathbb{E}_{s\sim\mathcal{D}}[D_{KL}(\pi||\pi_\beta)]\leq\varepsilon$. We resort to IQL to train the value net and Q-net as in Eqs.~\ref{eq_5} and~\ref{eq_6}. To extract the policy, we first calculate the cost advantage
\begin{equation}
    A^c(s,a)=Q_\psi^c(s,a)-V^c_\eta(s)
\end{equation}
and then apply AWR herein such that
\begin{equation}
    \pi_\text{safe} = \text{argmax}_{\pi}\mathbb{E}_{(s,a)\sim\mathcal{D}}[\text{exp}(\beta_cA^c(s,a)\text{log}\pi(a|s)],
\end{equation}
where $\beta_c\in[0,\infty)$ is a hyperparameter in AWR called inverse temperature corresponding to cost. To this end, the first-order stochastic optimization algorithm is used to improve the policy network by minimizing the corresponding cost policy loss:
\begin{equation}\label{eq_8}
    \theta \leftarrow \theta-\nu_c\nabla_\theta \mathcal{L}^c_\pi(\theta),
\end{equation}
where $\nu_c$ is the learning rate and $\mathcal{L}^c_\pi(\theta)$ is the policy loss associated with cost, which is expressed as
\begin{equation}
    \mathcal{L}^c_\pi(\theta)=\mathbb{E}_{(s,a)\sim\mathcal{D}}[-\text{exp}(-\beta_cA^c(s,a))\cdot\text{log}\pi_\theta(a|s)].
\end{equation}

\noindent\textbf{Performance maximization.} The phase of learning safety is as in most offline safe RL algorithms~\cite{xu2021crpo,koirala2024latent} where a feasible policy set is identified first. In the performance maximization phase, they either separately trained another model to extract the optimal policy respecting the learned safety, or directly optimized a policy and projected back to the feasible set. While the former complicates the learning process, the latter may lead to overly conservative behaviors. Instead, we fine-tune the model $\theta$ learned from the previous safety learning phase by assuming that the model has preserved safety requirements. This should empirically hold as we have minimized the same cumulative cost $\mathbb{E}_{\tau\sim\pi}[C(\tau)]$. The same dataset $\mathcal{D}$ is reused for maximizing performance, so the proximity constraint is still required.
Thereby, IQL is utilized to learn the optimal policy, by only fine-tuning $\theta$. Similarly, the value net and Q-net are updated by minimizing the losses in Eq.~\ref{eq_3} and~\ref{eq_4} respectively. To extract the policy, we have
\begin{equation}\label{lambda_update}
\begin{split}
\pi_\text{perf} = \arg\max_{\pi} \mathbb{E}_{(s,a)\sim\mathcal{D}}& \big[ -\exp( \beta_r ( A^r(s,a) \\
& + \tfrac{\lambda}{\beta_r} A^c(s,a) )) \log\pi(a|s) \big],
\end{split}
\end{equation}
where $A^r(s,a)$ is the reward advantage, $\beta_r$ is the inverse temperature corresponding to reward, and $\lambda$ is a regularization coefficient for cost due to the optimal policy search in the safe region. The role of $\lambda$ is critical as it penalizes the high cost when maximizing the performance. This also addresses the issue of catastrophic forgetting by switching from cost minimization to reward maximization.
One can manually tune this parameter to find a nearly optimal value in implementation, but it can be case-sensitive and time-consuming. Instead, we can resort to the general Lagrangian multiplier update for $\lambda$ such that
$
    \lambda\leftarrow\text{max}\{0,\lambda+\nu_\lambda(\hat{C}(\tau)-\kappa)\}$,
where $\nu_\lambda$ is the learning rate, and $\hat{C}(\tau)$ is the unbiased estimate of $\mathbb{E}_{\tau\sim\pi}[C(\tau)]$, which can be implemented by calculating from a randomly sampled batch of $\mathcal{D}$. For the policy update, we have
\begin{equation}\label{eq_11}
    \theta\leftarrow\theta-\nu_r\nabla_{\theta}\mathcal{L}^r_\pi(\theta),
\end{equation}
where $\theta$ represent the updated parameters during performance maximization phase, $\nu_r$ is the learning rate,
$\mathcal{L}^r_\pi(\theta)$ is the policy loss in the following form:
\begin{equation}
\begin{split}
\mathcal{L}^r_\pi(\theta) = \mathbb{E}_{(s,a)\sim\mathcal{D}} \big[ &\exp( \beta_r ( A^r(s,a) \\
& - \tfrac{\lambda}{\beta_r} A^c(s,a) )) \log\pi_{\theta}(a|s) \big].
\end{split}
\end{equation}
Note that though we use different notations to differentiate the optimization phases, the losses are calculated from the same model. We summarize the two-phase learning in lexicographic order in Algorithm~\ref{alg:algorithm}, referred to as LexiSafe-SC.

\subsection{Theoretical Analysis}
\noindent\textbf{Safety constraint violation bound.} This work aims to search for an safe optimal policy $\pi\in (\mathcal{S}\to\Delta(\mathcal{A}))$, where $\Delta(\cdot)$ is a probability simplex, with the help of Q-function class $\mathcal{F}_\nu\subset (\mathcal{S}\times \mathcal{A}\to[0,Q^\nu_m]), Q^\nu_m=\frac{\nu_m}{1-\gamma}, \nu\in\{r,c\}$. We assume that $\mathcal{F}_\nu$ is rich enough such that for any policy $\pi$, $Q^\pi_\nu\in\mathcal{F}_\nu$, where $Q^\pi_\nu$ is the true Q-function. This assumption is slightly stronger than the reachability assumption in~\cite{xie2021bellman}, as ours results in zero approximation error due to our primarily focus on the dominance of Bellman residual caused by the Q-function approximation. Additionally, $\mathcal{F}_\nu$ is also assumed to be a smooth function class in actions.
We also follow the standard $\mathcal{O}(\cdot)$ notation: $E=\mathcal{O}(F)$ is defined as $E\leq GF$ for some absolute constant $G>0$. The tilde notation $E=\tilde{\mathcal{O}}(F)$ denotes $E\leq GZ\cdot F$ where $Z$ is a poly-logarithmic factor of problem parameters. 

\begin{definition}\label{definition_2}
    A Bellman operator $\mathcal{T}_\nu^\pi:\mathbb{R}^{|\mathcal{S}|\times|\mathcal{A}|}\to\mathbb{R}^{|\mathcal{S}|\times|\mathcal{A}|}$ is defined as:
    \begin{equation}
        \begin{split}
            \forall h\in\mathcal{F}_\nu, (\mathcal{T}^\pi_\nu h)&(s,a) = \nu(s,a)+\\ &\gamma\mathbb{E}_{s'\sim\mathcal{P}(\cdot|s,a)}
            \times \mathbb{E}_{a'\sim\pi(\cdot|s')}[h(s',a')],
        \end{split}
    \end{equation}
    where $\nu\in\{r,c\}$.
\end{definition}

\begin{assumption}\label{assumption_1}
    The offline dataset $\mathcal{D}$ covers the state-action space in the sense that for any policy $\pi$ that we consider, there exists a constant $1\leq\mathcal{C}<\infty$ such that $d^\pi(s,a)\leq \mathcal{C}\cdot d^{\pi_\beta}(s,a)$. 
\end{assumption}
Assumption~\ref{assumption_1} resembles the popular concentrability coefficient in offline RL~\cite{xie2021policy,rashidinejad2021bridging} and characterizes the distance between the visitation distributions of the
behavior policy and the learned policy. Intuitively, $\mathcal{C}$ quantifies how much the state-action distribution induced by a learned policy $\pi$ can diverge from the behavior policy $\pi_\beta$.

With abuse of notation, we use a unified way to redefine the expected discounted return and cost: $J^\pi_\nu=\mathbb{E}_\pi[\sum_{t=0}^T\gamma^t\nu(s_t,a_t)]$, where $\nu\in\{r,c\}$. Based on Definition~\ref{definition_1}, the primary objective is to encode the safety into the model from dataset $\mathcal{D}$, implicitly yielding a feasible region, on top of which the optimal policy in the secondary objective ensures the safety. Thereby, our first main result in what follows explicitly shows the safety constraint violation bound.

We will leverage Q-function to evaluate the policy performance to quantify the safety violation and performance suboptimality bounds.
Although the critic is trained offline, its target $Q^\pi_c$ depends on the state-action distribution induced by the actor network. This reflects \textit{the interaction between actor expressiveness and critic generalization.} We next present a key lemma to reveal the Q-function estimation error.
\begin{lemma}\label{lemma_1}
    Let Assumption~\ref{assumption_1} hold and $\mathcal{F}_\nu$ be a function class of neural networks with VC dimension VCdim($\mathcal{F}_\nu$). For $\hat{Q}_\nu \in\mathcal{F}_\nu$ learned by IQL via empirical Bellman backup using dataset $\mathcal{D}$, with probability at least $1-\varrho$ ($\varrho>0$), we have the following relationship $
        \text{sup}_{s,a}|Q^\pi_\nu-\hat{Q}_\nu|\leq\mathcal{O}(\frac{\nu_m\sqrt{\mathcal{C}}}{1-\gamma}\sqrt{\frac{VCdim(\mathcal{F}_\nu)+\textnormal{log}\frac{1}{\varrho}}{|\mathcal{D}|}})$,
    where $Q^\pi_\nu$ and $\hat{Q}_\nu$ are respectively the true and estiamated Q-functions.
\end{lemma}
\begin{proof}
Let the true Bellman target be:
\begin{equation}
    y^\pi_\nu(s,a):=\nu(s,a)+\gamma\mathbb{E}_{s'\sim\mathcal{P},a\sim\pi(\cdot|s')}[Q^\pi_\nu(s',a')]
\end{equation}
Hence, the empirical Bellman error on dataset $\mathcal{D}$ is
\begin{equation}
    \mathcal{L}(Q_\nu)=\frac{1}{|\mathcal{D}|}\sum_{i=1}^{|\mathcal{D}|}(Q_\nu(s_i,a_i)-y^\pi_\nu(s_i,a_i))^2.
\end{equation}
As we don't have access to $y^\pi_\nu(s,a)$, we use sampled Bellman targets:
\begin{equation}
    \hat{y}_i:=\nu_i+\gamma\mathbb{E}_{a'\sim\pi(\cdot|s'_i)}[Q_\nu(s'_i,a')]. 
\end{equation}
Thus, the training objective is
\begin{equation}
   \hat{\mathcal{L}}(Q_\nu)=\frac{1}{|\mathcal{D}|}\sum_{i=1}^{|\mathcal{D}|}(Q_\nu(s_i,a_i)-\hat{y}_i)^2. 
\end{equation}
We fit $\hat{Q}_\nu=\text{argmin}_{Q\in\mathcal{F}_\nu}\hat{\mathcal{L}}(Q_\nu)$.
For a true Q-function, it satisfies the Bellman operator such that we have $Q^\pi_\nu=\mathcal{T}^\pi_\nu Q^\pi_\nu$. Next, we decompose the estimation error. Since we want to bound $\text{sup}_{s,a}|Q^\pi_\nu-\hat{Q}_\nu|$ such that we define the pointwise error $e(s,a)=Q^\pi_\nu-\hat{Q}_\nu$. From the contraction property of the Bellman operator $\mathcal{T}^\pi_\nu$: for any two Q-functions $Q_1,Q_2, \|\mathcal{T}^\pi_\nu Q_1-\mathcal{T}^\pi_\nu Q_2\|_\infty\leq\gamma\|Q_1-Q_2\|_\infty$, we can obtain the following relationship
\begin{equation}
    \begin{split}
        \|Q^\pi_\nu-\hat{Q}_\nu\|_\infty&=\|\mathcal{T}^\pi_\nu Q^\pi_\nu-\hat{Q}_\nu\|_\infty\\&=\|\mathcal{T}^\pi_\nu\hat{Q}_\nu-\hat{Q}_\nu+\mathcal{T}^\pi_\nu Q^\pi_\nu-\mathcal{T}^\pi_\nu\hat{Q}_\nu\|_\infty\\&\leq\|\mathcal{T}^\pi_\nu\hat{Q}_\nu-\hat{Q}_\nu\|_\infty+\|\mathcal{T}^\pi_\nu Q^\pi_\nu-\mathcal{T}^\pi_\nu\hat{Q}_\nu\|_\infty.
    \end{split}
\end{equation}
Rearranging this yields:
\begin{equation}\label{eq_19}
    \|Q^\pi_\nu-\hat{Q}_\nu\|_\infty\leq\frac{1}{1-\gamma}\|\mathcal{T}^\pi_\nu\hat{Q}_\nu-\hat{Q}_\nu\|_\infty.
\end{equation}
Based on Assumption~\ref{assumption_1}, we now have
\begin{equation}\label{eq_20}
   \|\mathcal{T}^\pi_\nu\hat{Q}_\nu-\hat{Q}_\nu\|_\infty\leq\sqrt{\mathcal{C}\mathbb{E}_{(s,a)\sim d^{\pi_\beta}}[(\mathcal{T}^\pi_\nu\hat{Q}_\nu(s,a)-\hat{Q}_\nu(s,a))^2]} 
\end{equation}
This is the Bellman residual to evaluate how well $\hat{Q}_\nu$ satisfies the Bellman equation, which can be bounded below. Since $\hat{Q}_\nu$ is the minimizer of the empirical loss, we can bound the difference between empirical and population Bellman errors using uniform convergence. Let
\begin{equation}
    \mathcal{E}(Q_\nu):=\mathbb{E}_{(s,a)\sim\mathcal{D}}[(Q_\nu(s,a)-\mathcal{T}^\pi_\nu Q_\nu(s,a))^2]
\end{equation}
\begin{equation}
    \tilde{\mathcal{E}}(Q_\nu):=\frac{1}{N}\sum_{i=1}^N(Q_\nu(s_i,a_i)-\hat{y}_i)^2
\end{equation}
Based on a well-known result in VC dimension (uniform convergence theory for regression with squared loss)~\cite{bartlett2002rademacher,chen2019information,munos2008finite} and applying to Bellman residual minimization in RL, with probability at least $1-\varrho$, we have
\begin{equation}
    \begin{split}
        &\text{sup}_{Q_\nu\in\mathcal{F}_\nu}|\mathcal{E}(Q_\nu)- \tilde{\mathcal{E}}(Q_\nu)|\\&\leq \mathcal{O}(\frac{(\nu_m)^2(VCdim(\mathcal{F}_\nu)+\text{log}(1/\varrho))}{|\mathcal{D}|})
    \end{split}
\end{equation}
Due to the rich function class assumption on $\mathcal{F}_\nu$, we know
\begin{equation}
\begin{split}
   &\mathbb{E}_{(s,a)\sim d^{\pi_\beta}}[(\mathcal{T}^\pi_\nu\hat{Q}_\nu(s,a)-\hat{Q}_\nu(s,a))^2]\\&\leq \mathcal{O}(\frac{(\nu_m)^2(VCdim(\mathcal{F}_\nu)+\text{log}(1/\varrho))}{|\mathcal{D}|}).
   \end{split}
\end{equation}
Substituting the last relationship into Eq.~\ref{eq_20}, and then plugging the updated Eq.~\ref{eq_20} into Eq.~\ref{eq_19} completes the proof.
\end{proof}
In Lemma~\ref{lemma_1}, the concentrability coefficient appears as a multiplicative factor in the Q-function approximation error bound. Intuitively, if $d^\pi$ differs greatly from $d^{\pi_\beta}$, then errors in regions poorly covered by the dataset can be amplified when evaluated under $\pi$. Therefore, a large $\mathcal{C}$ indicates poor coverage and leads to higher estimation error, whereas small $\mathcal{C}$ reflects better alignment and more reliable generalization of the Q-function to unseen states and actions. In Lemma~\ref{lemma_1}, we defined a VC dimension for the function class $\mathcal{F}_\nu$, but are not aware of how the dimensions of model architecture explicitly affect the bound. Specifically, we recall a well-known result to concretely define the relationship for VC dimension from~\cite{anthony2009neural}: VCdim($\mathcal{F}_\nu)=\mathcal{O}(D\cdot L\cdot\text{log}D$), where $D$ is the number of parameters and $L$ is the number of layers. Although VC-dimension–based bounds can be loose, it remains a standard and broadly applicable measure for analyzing generalization in RL. It provides distribution-independent guarantees, enables clear links between model complexity and safety/performance errors, and allows consistent comparison with prior offline and safe RL analyses. Thus, we adopt VC dimension as a principled and general capacity measure despite its potential conservativeness.

\begin{theorem}\label{theorem_1}
    Let Assumption~\ref{assumption_1} hold and $\pi_\text{safe}$ be the policy obtained by minimizing $J^\pi_c$ using IQL with a neural network of depth $L$ and number of parameters $d_\theta$ by using dataset $\mathcal{D}$. Suppose that $Q^\pi_c\in\mathcal{F}_c$. With probability at least $1-\varrho$, for any constant $\mathcal{C}_{safe}\geq 1$, we have the following relationship: $
        J^{\pi_\text{safe}}_c\leq \kappa+\mathcal{O}\bigg(\frac{c_m\sqrt{\mathcal{C}_{safe}}\sqrt{d_\theta L\text{log}(d_\theta)+\text{log}(1/\varrho)}}{(1-\gamma)^2\sqrt{|\mathcal{D}|}}\bigg)$.
\end{theorem}
\begin{proof}
Let the optimal safe policy be $\pi^*_{safe}$ such that based on the well-known Performance Difference Lemma~\cite{kakade2002approximately,schulman2015trust}, we can obtain
\begin{equation}
    J^{\pi_\text{safe}}_c - J^{\pi^*_safe}_c = \frac{1}{1-\gamma}\mathbb{E}_{(s,a)\sim d^{\pi_\text{safe}}}[A^{\pi^*}_c(s,a)],
\end{equation}
where $A^{\pi^*}_c(s,a)=Q^{\pi^*}_c(s,a)-V^{\pi^*}_c(s)$.
To make this expression tractable, we add and subtract the estimated Q-function:
\begin{equation}
    \begin{split}
        A^{\pi^*}_c(s,a)&=\hat{Q}_c(s,a)-\hat{V}_c(s) + [Q^{\pi^*}_c(s,a) - \hat{Q}_c(s,a)] \\& - [V^{\pi^*}_c(s) - \hat{V}_c(s)].
    \end{split}
\end{equation}
Thus, taking expectation over $d^{\pi_\text{safe}}$ and using Triangle inequality, we have
\begin{equation}
    \begin{split}
        |A^{\pi^*}_c(s,a)|&\leq|Q^{\pi^*}_c(s,a) - \hat{Q}_c(s,a)| \\& + |V^{\pi^*}_c(s) - \hat{V}_c(s)|.
    \end{split}
\end{equation}
Applying Cauchy-Schwartz inequality yields 
\begin{equation}\label{eq_29}
    \begin{split}
        |J^{\pi_\text{safe}}_c - J^{\pi^*}_c|&\leq \frac{1}{1-\gamma}\mathbb{E}_{(s,a)\sim d^{\pi_\text{safe}}} [|A^{\pi^*}_c(s,a)|]\\&\leq \frac{1}{1-\gamma}\bigg(\|\hat{Q}_c-Q^{\pi^*}_c\|_\infty+\|\hat{V}_c-V^{\pi^*}_c\|_\infty\bigg).
    \end{split}
\end{equation}
Given the bound of $\|\hat{Q}_c-Q^{\pi^*}_c\|_\infty$, we next bound $\|\hat{V}_c-V^{\pi^*}_c\|_\infty$.
Recall $V^{\pi^*}_c=\mathbb{E}_{a\sim\pi^*(\cdot|s)}[Q^{\pi^*}_c(s,a)], \hat{V}_c=\mathbb{E}_{a\sim\pi^*(\cdot|s)}[\hat{Q}_c(s,a)]$. Thus, we have
\begin{equation}
    \begin{split}
        |\hat{V}_c-V^{\pi^*}_c|&=|\mathbb{E}_{a\sim\pi^*(\cdot|s)}[\hat{Q}_c(s,a)-Q^{\pi^*}_c(s,a)]|\\&\leq \mathbb{E}_{a\sim\pi^*(\cdot|s)}[|\hat{Q}_c(s,a)-Q^{\pi^*}_c(s,a)|].
    \end{split}
\end{equation}
Now take the squared expectation over $s\sim d^{\pi_\text{safe}}(s)$, we have:
\begin{equation}
    \begin{split}
        &\|\hat{V}_c-V^{\pi^*}_c\|^2_\infty=\mathbb{E}_{s\sim d^{\pi_\text{safe}}(s)}[(\mathbb{E}_{a\sim\pi^*(\cdot|s)}[\hat{Q}_c(s,a)-Q^{\pi^*}_c(s,a)])^2]\\&\leq \mathbb{E}_{s\sim d^{\pi_\text{safe}}(s)}[\mathbb{E}_{a\sim\pi^*(\cdot|s)}[(\hat{Q}_c(s,a)-Q^{\pi^*}_c(s,a)])^2]],
    \end{split}
\end{equation}

where the second inequality is based on Jensen's inequality. Since we have KL divergence as proximity constraint in our problem setup and $\mathcal{F}_\nu$ is smooth in actions, this implies that $\pi_\text{safe}$ and $\pi^*$ are close such that
\begin{equation}
    \|\hat{V}_c-V^{\pi^*}_c\|^2_\infty\leq \|\hat{Q}_c-Q^{\pi^*}_c\|^2_\infty.
\end{equation}
Thus, Eq.~\ref{eq_29} can be rewritten as
\begin{equation}
    |J^{\pi_\text{safe}}_c - J^{\pi^*}_c|\leq\frac{2}{1-\gamma}(\|\hat{Q}_c-Q^{\pi^*}_c\|_\infty).
\end{equation}
According to Lemma~\ref{lemma_1} and based on the fact that the optimal policy is constraint satisfactory, the desirable result is obtained.
\end{proof}

In Theorem~\ref{theorem_1}, the concentrability coefficient is adapted to the safety learning phase, reflecting the different policies across lexicographic stages. The safety violation gap mainly depends on model parameter dimension, concentrability, and dataset size. When $\pi_\text{safe}=\pi_\beta$, the bound becomes independent of $\mathcal{C}_{safe}$,  
indicating safe policy improvement. We next derive the performance suboptimality bound between the learned safe policy $\pi_\text{perf}$ and the true optimal policy $\pi^*$.

\noindent\textbf{Performance suboptimality bound.} As defined, the lexicographic order enables learning a safety-aware optimal policy in the simplex $\Delta(\mathcal{A})$, produced by the cost-trained actor. While the optimization in this phase is different from that in the last phase, updating the model is exactly identical between Eq.~\ref{eq_8} and Eq.~\ref{eq_11}. Analogously, we can obtain the following suboptimality bound for the performance maximization.

\begin{theorem}\label{theorem_2}
    Let Assumption~\ref{assumption_1} hold and $\pi_\text{perf}$ be the policy obtained by maximizing $J^\pi_r$ using IQL with a neural network of depth $L$ and the number of parameters $d_\theta$ by using dataset $\mathcal{D}$. Suppose that $Q^\pi_r\in\mathcal{F}_r$. With probability at least $1-\varrho$, for any constant $\mathcal{C}_{perf}\geq 1$, we have the following relationship: $
    J^{\pi^*}_r-J^{\pi_\text{perf}}\leq \mathcal{O}\bigg(\frac{r_m\sqrt{\mathcal{C}_{perf}}\sqrt{d_\theta L\text{log}(d_\theta)+\text{log}(1/\varrho)}}{(1-\gamma)^2\sqrt{|\mathcal{D}|}}\bigg)$.
\end{theorem}
\begin{proof}
    Similar to the proof for Theorem~\ref{theorem_1}, based on performance difference lemma, we can get the following relationship
    \begin{equation}
        J^{\pi^*}_r-J^{\pi_\text{perf}}\leq\frac{2}{1-\gamma}\|Q^{\pi^*}_r-\hat{Q}_r\|_\infty
    \end{equation}
    By classical generalization bounds for predictors in~\cite{bartlett2002rademacher}, we can know that 
    \begin{equation}
        \|Q^{\pi^*}_r-\hat{Q}_r\|_\infty\leq \mathcal{O}\bigg(\frac{r_m\sqrt{\mathcal{C}_{perf}}}{1-\gamma}\sqrt{\frac{d_\theta L\text{log}(d_\theta)+\textnormal{log}\frac{1}{\varrho}}{|\mathcal{D}|}}\bigg).
    \end{equation}
    Combining the last two inequalities attains the desirable result.
\end{proof}
Theorem~\ref{theorem_2} shows that performance suboptimality mainly depends on model dimensions, dataset size, and concentrability. 

In the sequel, we establish the sample complexity for LexiSafe-SC.

\noindent\textbf{Sample Complexity bound.}
Denote the number of samples $N=|\mathcal{D}|$ and suppose that both safety violation and performance suboptimality bounds are ensured to be less than or equal to a desirable accuracy ($\leq \epsilon$) for some sufficiently small $\epsilon>0$. 
\begin{theorem}\label{theorem_3}
    Let $T:=\frac{1}{1-\gamma}$ be the effective horizon. Given a desired accuracy $\epsilon$, with probability at least $1-\varrho$, to learn a safe optimal policy $\pi_\text{perf}$, the sample copmlexity incurred by LexiSafe-SC is
    $
        N=\tilde{\mathcal{O}}(\frac{T^4}{\epsilon^2}\cdot\text{max}\{c^2_m\mathcal{C}_{safe}d_{\theta}L, r^2_m\mathcal{C}_{perf}d_\theta L\})$.
\end{theorem}
\begin{proof}
    Based on the conclusions from Theorem~\ref{theorem_1} and Theorem~\ref{theorem_2}, we let
    \begin{equation}
        J^{\pi_\text{safe}}_c\leq\kappa+\epsilon\;,J^{\pi^*}-J^{\pi_\text{perf}}\leq \epsilon. 
    \end{equation}
    Hence, it is immediately to obtain the relationship between $N$ and $\epsilon$. By ignoring the poly-logarithmic factors with respect to $d_{\theta}, \varrho$, we complete the proof.
\end{proof}
The sample complexity of LexiSafe depends on model size, distribution shift, and the horizon. Larger actor networks require more data, while high concentrability and long horizons increase sample demands due to distribution mismatch and credit assignment challenges. By far, we have analyzed the sampling complexity for the single-cost scenario. In what follows, we will extend the LexiSafe-SC to the scenario with multiple costs and present the associated sample complexity.

\subsection{Generalization to Multiple Costs}
Eq.~\ref{eq_2} presents the offline safe RL formulation without any hierarchical structure within the cost. However, many real-world applications feature not only a hierarchy between safety and reward, but also additional layers of hierarchy within the safety objective itself. For example, in autonomous driving~\cite{yang2024genesis}, one must balance competing safety considerations such as avoiding collisions versus complying with traffic regulations, which inevitably requires prioritizing different safety rules. This naturally leads to a multi-phase cost-minimization problem in safety learning. Consequently, LexiSafe-SC must be extended to handle the multi-cost setting. To this end, we first adapt Eq.~\ref{eq_2} to the following:
\begin{equation}
    \text{max}_\pi V^\pi_r(s), \; s.t., V^\pi_{c_j}(s)\leq \kappa_j,\; D_{KL}(\pi||\pi_\beta)\leq \varepsilon,
\end{equation}
where $j\in\{1,...,K-1\}$. In this context, we assume that there are totally $K$ phases of learning including the first $K-1$ phases of safety learning. Given this in hand, the lexicographic order for the multi-cost scenario can be redefined as:
\begin{definition}\label{definition_3}
    Given a pre-collected dataset $\mathcal{D}$ from some unknown behavior policy $\pi_\beta$, LexiSafe-MC enforeces a strict hierarchical priorities between objectives:
    \begin{itemize}
        \item Primary objectives (safety):
        \begin{equation}\begin{split}&\pi^j_\text{safe}\in\text{argmin}_{\pi\in\Pi_{j-1}}\mathbb{E}_{\tau\sim\pi}[C_j(\tau)] \\&\text{s.t.,}\mathbb{E}_{s\sim\mathcal{D}}[D_{KL}(\pi||\pi_\beta)]\leq\varepsilon, \forall j=1,...,K-1,\end{split}\end{equation}
        where $\Pi_0=\mathcal{A}, \Pi_k=\{\pi\in\Pi_{j-1}|\mathbb{E}_{\tau\sim\pi}[C_j(\tau)]\leq\kappa_j\}$.
        \item Secondary objective (performance): 
        \begin{equation}\begin{split}\pi^*=\text{argmax}_{\pi\in\Pi_{K-1}}&\mathbb{E}_{\tau\sim\pi}[R(\tau)] \\& \text{s.t.,}\mathbb{E}_{s\sim\mathcal{D}}[D_{KL}(\pi||\pi_\beta)]\leq\varepsilon.\end{split}\end{equation}
    \end{itemize}
\end{definition}
To solve the above optimization problem, we still adopt IQL to learn the optimal policy by recursively training $\theta$ to comply with the safety priorities. 
Similarly, multiple value nets and Q-nets are updated by minimizing the losses in Eq.~\ref{eq_5} and~\ref{eq_6} respectively. To extract the policy in each phase of cost minimization, we have
\begin{equation}
    \pi^j_{safe} = \text{argmax}_{\pi_j}\mathbb{E}_{(s,a)\sim\mathcal{D}}[\text{exp}(\beta^j_cA^c_j(s,a)\text{log}\pi_j(a|s)].
\end{equation}
where $A^c_j(s,a)$ is the cost advantage corresponding to phase $j$, $\beta_c^j\in[0,\infty)$ is a hyperparameter in AWR called inverse temperature corresponding to cost. Similarly, the first-order optimization algorithm is used to improve the policy network by minimizing the corresponding cost policy loss in phase $j$:
\begin{equation}\label{eq_15}
    \theta \leftarrow \theta-\nu_c\nabla_\theta \mathcal{L}^c_{\pi,j}(\theta),
\end{equation}
where $\nu_c$ is the learning rate and $\mathcal{L}^c_{\pi,j}(\theta)$ is the policy loss associated with cost $j$, which is expressed as
\begin{equation}
    \mathcal{L}^c_{\pi,j}(\theta)=\mathbb{E}_{(s,a)\sim\mathcal{D}}[-\text{exp}(-\beta^j_cA^c_j(s,a))\cdot\text{log}\pi_\theta(a|s)].
\end{equation}
Following the multi-phase learning for safety, the reward maximization aims to improve the policy as follows:
\begin{equation}\label{eq_17}
\begin{split}
   \pi_\text{perf} &= \text{argmax}_{\pi}\mathbb{E}_{(s,a)\sim\mathcal{D}}[\text{exp}(\beta_r(A^r(s,a)\\&-\sum_{j=1}^{K-1}\frac{\lambda_j}{\beta_r}A^c_j(s,a)))\text{log}\pi_j(a|s)],
\end{split}
\end{equation}
where $\lambda_j$ is a regularization coefficient for cost due to the optimal policy search in the $j$-th phase.
Likewise, if using first-order method to update the parameter, we can have the following:
\begin{equation}\label{eq_18}
    \theta\leftarrow\theta-\nu_r\nabla_{\theta}\mathcal{L}^r_\pi(\theta),
\end{equation}
where $\mathcal{L}^r_\pi(\theta)$ is the policy loss in the following form:
\begin{equation}
    \begin{split}
        \mathcal{L}^r_\pi(\theta)=\mathbb{E}_{(s,a)\sim\mathcal{D}}&[\text{exp}(\beta_r(A^r(s,a)-\\&\sum_{j=1}^{K-1}\tfrac{\lambda_j}{\beta_r}A^c(s,a)))\text{log}\pi_{\theta}(a|s)].
    \end{split}
\end{equation}
Since in each phase of cost minimization, the update is similar to that in the single-cost scenario, the associated cost constraint violation bound is analogous to the conclusion in Theorem~\ref{theorem_1}. Therefore, we have the following corollary for the sample complexity for the multi-cost scenario.
\begin{corollary}\label{corollary_1}
    Let $T:=\frac{1}{1-\gamma}$ be the effective horizon. Given a desired accuracy $\epsilon$, with probability at least $1-\varrho$, to learn a safe optimal policy $\pi_\text{perf}$, the sample copmlexity incurred by LexiSafe-MC is
$
    N=\tilde{\mathcal{O}}(\frac{T^4}{\epsilon^2}\cdot\text{max}\{(Kc^2_m\mathcal{C}_{safe}d_{\theta}L, r^2_m\mathcal{C}_{perf}d_\theta L\})$.
\end{corollary}
\begin{proof}
    We follow the same proof techniques adopted for Theorem~\ref{theorem_3}. Since there are $K-1$ phases of cost minimization to correspond to safety learning, we can simply accumulate them. The sample complexity for the performance maximization remains the same. This completes the proof.
\end{proof}
Corollary~\ref{corollary_1} suggests that increasing the number of phases in the cost minimization process can lead to higher sample complexity, and the bound naturally reduces to that of the single-cost setting when $K=1$. Algorithm~\ref{alg:lexisafe_multi_single_theta_clean} presents the overall framework for LexiSafe-MC. To obtain more accurate cost estimates for each phase in the offline RL setting, we apply exponential smoothing (Lines 9–10) between the empirical return and the predicted value computed from a mini-batch $\mathcal{B}$. $\tilde J^j_c$ is the estimated cost return in the phase $j$. Lines 11–14 illustrate the cost minimization procedure for the $K-1$ safety-related phases, ensuring adherence to the prescribed safety priorities. In contrast, the method in~\cite{skalse2022lexicographic} requires carefully chosen step sizes that are non-summable yet square-summable, complicating their algorithm and slowing convergence. LexiSafe-MC, however, retains the simplicity of using constant step sizes for all updates.
\begin{algorithm}[t]
\caption{LexiSafe-MC, multiple costs single reward}
\label{alg:lexisafe_multi_single_theta_clean}
\begin{algorithmic}[1]
\STATE \textbf{Initialization}: $\rho,\phi$ for reward $Q$ and $V$ nets;\; $\{\,\psi_j,\eta_j\,\}_{j=1}^{K-1}$ for cost $Q$ and $V$ nets;\; policy parameters $\theta$;\; Lagrange multipliers $\{\,\lambda_j\,\}_{j=1}^{K-1} \ge 0$;\; cost thresholds $\{\,\kappa_j\,\}_{j=1}^{K-1}$;\; learning rates $\nu_c,\nu_r,\nu_\lambda$;\; discount $\gamma$;\; moving averge constant  $\alpha\in[0,1]$.
\FOR{each gradient step}
  \STATE \texttt{Sample minibatch} $\mathcal{B}$
  \STATE \texttt{Update reward critics:}
  \STATE \hspace{0.75em} $\rho \leftarrow \rho - \nu_r\nabla_\rho\,\mathcal{L}^{r}_Q(\rho)$\\
    \hspace{0.75em} $\phi \leftarrow \phi - \nu_r\nabla_\phi\,\mathcal{L}^{r}_V(\phi)$
  \STATE \texttt{Update cost critics for each $j$:}
  \STATE \hspace{0.75em} $\psi_j \leftarrow \psi_j - \nu_c\nabla_{\psi_j}\,\mathcal{L}^{c}_{Q_j}(\psi_j)$
         \\
         \hspace{0.75em} $\eta_j \leftarrow \eta_j - \nu_c\nabla_{\eta_j}\,\mathcal{L}^{c}_{V_j}(\eta_j)$
  \STATE \texttt{Estimate costs and smoothing:}
  \STATE \hspace{0.75em} $\widehat C_j \;\leftarrow\; \frac{1}{|\mathcal B|}\sum_{(s,\cdot)\in\mathcal B} V^c_{\eta_j}(s)$
  \STATE \hspace{0.75em} $\tilde C_j \leftarrow (1-\alpha)\,\tilde J^j_c + \alpha\,\widehat C_j$

  \STATE \texttt{Phase 1,..,K-1: Cost minimization, for each $j$:}

  \STATE \hspace{0.75em} $\lambda_j\leftarrow\text{max}\{0,\lambda_j+\nu_\lambda(\tilde{C_j}-\kappa_j)\}$

  \STATE \hspace{0.75em} $A^c_j(s,a)\leftarrow Q^c_{\psi_j}(s,a)-V^c_{\eta_j}(s)$
  \STATE \hspace{0.75em} $\theta\leftarrow\theta-\nu_c\nabla_\theta\mathcal{L}^c_{\pi,j}(\theta)$
  \STATE \texttt{Phase K: Reward maximization}
  \STATE \hspace{0.75em} $A^r(s,a)\leftarrow Q_\rho(s,a)-V_\phi(s)$
  \STATE \hspace{0.75em} $\theta \leftarrow\theta-\nu_r\nabla_{\theta_1}\mathcal{L}^r_\pi(\theta)$ 
\ENDFOR
\STATE \textbf{Output}: $\pi_\theta$
\end{algorithmic}
\end{algorithm}

\section{Numerical Results}
We evaluate both variants of LexiSafe on simulators to assess: (1) performance against state-of-the-art offline safe baselines; (2) the benefits of imposing lexicographic structure versus a flat (weighted) cost-reward objective. In what follows, we detail the dataset, model architecture and hyperparameters, the hardware and experiment setup, and the results. 

\noindent\textbf{Dataset.}
Experiments are conducted on dataset selected for their relevance to safe offline reinforcement learning benchmarking, specifically the DSRL benchmark~\cite{liu2023datasets}. This benchmark includes trajectories drawn from tasks in the MetaDrive~\cite{li2021metadrive}, Bullet Safety Gym~\cite{gronauer2022bullet}, and Safety Gymnasium~\cite{ji2023safety} environments. In our work, we select a subset of these tasks to demonstrate the effectiveness of our method. 

\noindent\textbf{Model Architecture and Hyperparameters.}
We follow the standard setup as widely used in RL domain to parameterize the actor and critic networks. Specifically, they are all multi-layer perceptron (MLP) models. The hyperparameter setting is shown in Table~\ref{table:hyper_params}. Note that in this context we summarize key hyperparameters in RL setting. Though a hyperparameter optimization method can likely be beneficial for the performance improvement, we tune them manually in this work.

\begin{table}[t]
\caption{Hyperparameters}
\begin{center}
\begin{threeparttable}
\begin{tabular}{c c}
    \toprule
    \textbf{Hyperparameter} & \textbf{Value (Experiment range)}\\ \midrule
     Batch size &  2048 (32-2048)\\
     Discount factor & 0.995 (0.9-0.995)\\
     $\beta_c$ and $\beta_r$ & 1 (1-5)\\
     $\xi_{reward}$ and $\xi_{cost}$ & 0.7 (0.6-0.7)\\
     Actor learning rate & 3e-4 (3e-4-1e-3)\\
     Q and value learning rate & 3e-5 (1e-5-3e-5)\\
     $\lambda_c$ learning rate & 1e-4 (1e-4-1e-3)\\
     Hidden layer dimension & 128 (4-256)\\
     IQL Q network soft update $\tau$ & 0.005\\
     Training random seeds & 7, 17, 27, 77, 777\\
     Testing random seeds & 14, 42, 84, 98, 49\\
      \bottomrule
\end{tabular}
\end{threeparttable}
\end{center}
\label{table:hyper_params}
\end{table}

\noindent\textbf{Comparative study with LexiSafe-SC.} We first compare LexiSafe-SC with several recent safe RL baselines (BC-Safe, COptiDICE, CPQ, FISOR, LSPC-O) across tasks from Safety Gymnasium~\cite{ji2023safety} and Bullet Safety Gym~\cite{gronauer2022bullet}. Experiments follow the DSRL benchmark protocol~\cite{liu2023datasets}: for each environment, we train LexiSafe-SC with five random seeds, and for each trained agent we perform ten evaluation runs, each with five additional seeds, yielding 250 trajectories in total. We compute the raw mean and variance over these trajectories. Performance is reported using normalized reward $R=(R_\pi-R_{min})/(R_{max}-R_{min})$ and normalized cost $C=C_\pi/\kappa$, where $R_\pi$ is the undiscounted total per-episode reward, $R_{max}, R_{min}$ are task-specific constants, and $\kappa>0$ is the target cost threshold. A policy is considered safe when $C<1$; for safety eveluation we therefore only consider whether $C$ falls below 1. Its exact magnitude below 1 is not relevant to feasibility.

In Table~\ref{tab:performance_large_table}, safe agents ($C<1$) are shown in boldface. Entries highlighed in blue mark the safe agent(s) achieving the highest reward (we also mark agents within 0.05 of the top safe reward in blue). As shown in Table~\ref{tab:performance_large_table}, LexiSafe achieves state-of-the-art performance while adhering to safety constraints. BC-Safe~\cite{liu2023datasets} relies on behavior cloning from filtered safe data, but its effectiveness is limited by the quantity of safe samples and it must be retrained for different thresholds. COptiDICE~\cite{lee2022coptidice} performs the worst, suffering from inaccurate distribution correction estimates. CPQ~\cite{xu2022constraints} penalizes OOD actions using a conditional variational autoencoder, but this can distort value estimates and hinder generalization. FISOR~\cite{zheng2024safe} ensures safety via hard constraints but often leads to overly conservative policies. LSPC-O~\cite{koirala2024latent} is competitive but uses separate models for safety and performance optimization. In contrast, LexiSafe employs a lexicographic optimization framework that sequentially prioritizes safety before performance. This decoupled yet integrated training structure enables stable learning, constraint satisfaction, and preserves adaptability to performance improvement. Unlike BC-Safe, LexiSafe is robust to suboptimal demonstrations, and compared to FISOR, CPQ, and COptiDICE, it avoids the brittleness of joint constrained optimization. Moreover, LexiSafe improves upon LSPC-O by achieving comparable safety–performance trade-offs with a simpler architecture and and provable bounds, making it well-suited for safety-critical offline RL.

\begin{table*}[h]
    \centering
    \caption{Comparison of our methods with baselines across benchmark tasks. \textbf{Bold} indicates safety (cost $<$ 1), and \textbbf{blue} denotes both safety and high performance. Polices that are safe and is within 0.05 range of the highest reward are also marked blue.} \label{tab:performance_large_table}
    \resizebox{1.0\textwidth}{!}{
    \rowcolors{2}{green!5}{purple!5}
    \begin{tabular}{|c|cc|cc|cc|cc|cc|cc|}
        \hline
        \textbf{Method} & \multicolumn{2}{c}{\textbf{BC-Safe}} & \multicolumn{2}{c}{\textbf{COptiDICE}}  & \multicolumn{2}{c}{\textbf{CPQ}} & \multicolumn{2}{c}{\textbf{FISOR}} & \multicolumn{2}{c}{\textbf{LSPC-O}} & \multicolumn{2}{c|}{\textbf{LexiSafe-SC}}\\
        \hline
        Task & reward $\uparrow$ & cost $\downarrow$ & reward $\uparrow$ & cost $\downarrow$ & reward $\uparrow$ & cost $\downarrow$ & reward $\uparrow$ & cost $\downarrow$ & reward $\uparrow$ & cost $\downarrow$ & reward $\uparrow$ & cost $\downarrow$ \\
        \hline
        \multicolumn{1}{c}{Safety Gym:} \\
        \hline
        SwimmerVel & 0.51 & 1.07 & 0.63 & 7.58 & 0.13 & 2.66 &  \textbf{-0.04} & \textbf{0.00} & \textbf{0.44} & \textbf{0.14} & \textbbf{0.51} & \textbbf{0.92} \\
        HopperVel & \textbf{0.36} & \textbf{0.67} & 0.13 & 1.51 & 0.14 & 2.11 & \textbf{0.17} & \textbf{0.32} & \textbbf{0.69} & \textbbf{0.00} & \textbbf{0.70} & \textbbf{0.50} \\
        HalfCheetahVel & \textbf{0.88} & \textbf{0.54} & \textbf{0.65} & \textbf{0.00} & \textbf{0.29} & \textbf{0.74} & \textbf{0.89} & \textbf{0.00} & \textbbf{0.97} & \textbbf{0.10} & \textbbf{0.97} & \textbbf{0.69}\\
        Walker2dVel & \textbbf{0.79} & \textbbf{0.04} & \textbf{0.12} & \textbf{0.74} & \textbf{0.04} & \textbf{0.21}  & \textbf{0.38} & \textbf{0.36} & \textbbf{0.76} & \textbbf{0.02} & \textbbf{0.78} & \textbbf{0.40} \\
        AntVel & \textbbf{0.98} & \textbbf{0.29} & 1.0 & 3.28 & \textbf{-1.01} & \textbf{0.0} & \textbf{0.89} & \textbf{0.00} & \textbbf{0.98} & \textbbf{0.45} & \textbbf{0.98} & \textbbf{0.73} \\
        \hline
        \textbf{Average} & \textbf{0.70} & \textbf{0.14} & 0.51 & 2.62 & -0.08 & 1.14 &  \textbf{0.46} & \textbf{0.14} &  \textbbf{0.77} & \textbbf{0.15} & \textbbf{0.79} & \textbbf{0.65}\\

        \hline
        \multicolumn{1}{c}{Bullet Safety Gym:} \\
        \hline
        BallRun      & 0.27 & 1.46 & 0.59 & 3.52 & 0.22 & 1.27 & \textbbf{0.18} & \textbbf{0.00} & \textbf{0.14} & \textbf{0.00} & \textbbf{0.22} & \textbbf{0.18}\\
        CarRun       & \textbbf{0.94} & \textbbf{0.22} & \textbf{0.87} & \textbf{0.00} & 0.95 & 1.79 & \textbf{0.73} & \textbf{0.04} & \textbbf{0.97} & \textbbf{0.13} & \textbbf{0.98} & \textbbf{0.85}\\
        AntRun       & 0.65 & 1.09 & \textbf{0.61} & \textbf{0.94} &  \textbf{0.03} & \textbf{0.02} & \textbf{0.45} & \textbf{0.00} & \textbf{0.44} & \textbf{0.45} & \textbbf{0.65} & \textbbf{0.60} \\ 
        BallCircle   & \textbf{0.52} & \textbf{0.65} & 0.70 & 2.61 & \textbf{0.64} & \textbf{0.76} & \textbf{0.34} & \textbf{0.00} & \textbf{0.47} & \textbf{0.01} & \textbbf{0.71} & \textbbf{0.61}\\
        CarCircle    & \textbf{0.5}  & \textbf{0.84} & 0.49 & 3.14 & \textbbf{0.71} & \textbbf{0.33} & \textbf{0.40} & \textbf{0.03} & \textbbf{0.72} & \textbbf{0.04} & \textbbf{0.71} & \textbbf{0.58}\\
        DroneCircle  & \textbbf{0.56} & \textbbf{0.57} & 0.26 & 1.02 & -0.22 & 1.28 & \textbf{0.48} & \textbf{0.00} & \textbbf{0.58} & \textbbf{0.60} & \textbf{0.51} & \textbf{0.18}\\
        AntCircle   & \textbf{0.40}  & \textbf{0.96} & 0.17 & 5.04 & \textbf{0.00}  & \textbf{0.00} & \textbf{0.20} & \textbf{0.00} & \textbf{0.45} & \textbf{0.40} & \textbbf{0.51} & \textbbf{0.51}\\
        \hline
        \textbf{Average} & \textbf{0.55} & \textbf{0.94} & 0.53 & 2.32 & \textbf{0.33} & \textbf{0.78} &  \textbf{0.40} & \textbf{0.01} &  \textbf{0.54} & \textbf{0.23} & \textbbf{0.61} & \textbbf{0.50}\\
        \hline
    \end{tabular}
}
\end{table*}

\noindent\textbf{Ablation study with LexiSafe-MC.}  
The ablation study evaluates two aspects of LexiSafe-MC: (a) its ability to enforce multiple, hierarchically ordered safety constraints, and (b) its advantage compared to a flat, weighted objective implemented with vanilla IQL. We conduct experiments in the MetaDrive easydense environment, where individual costs are provided by the DSRL dataset and realistic driving-style safety costs can be naturally arranged into hierarchies, aligning with the intended application of LexiSafe-MC. We focus on two cost types: crash and speed, and evaluate two hierarchical orderings:
(1) Crash $\rightarrow$ Speed $\rightarrow$ Reward, prioritizing crash safety, and (2) Speed $\rightarrow$ Crash $\rightarrow$ Reward, prioritizing speed regulation.

\begin{figure}[b!]
    \centering
    \includegraphics[width=0.9\linewidth]{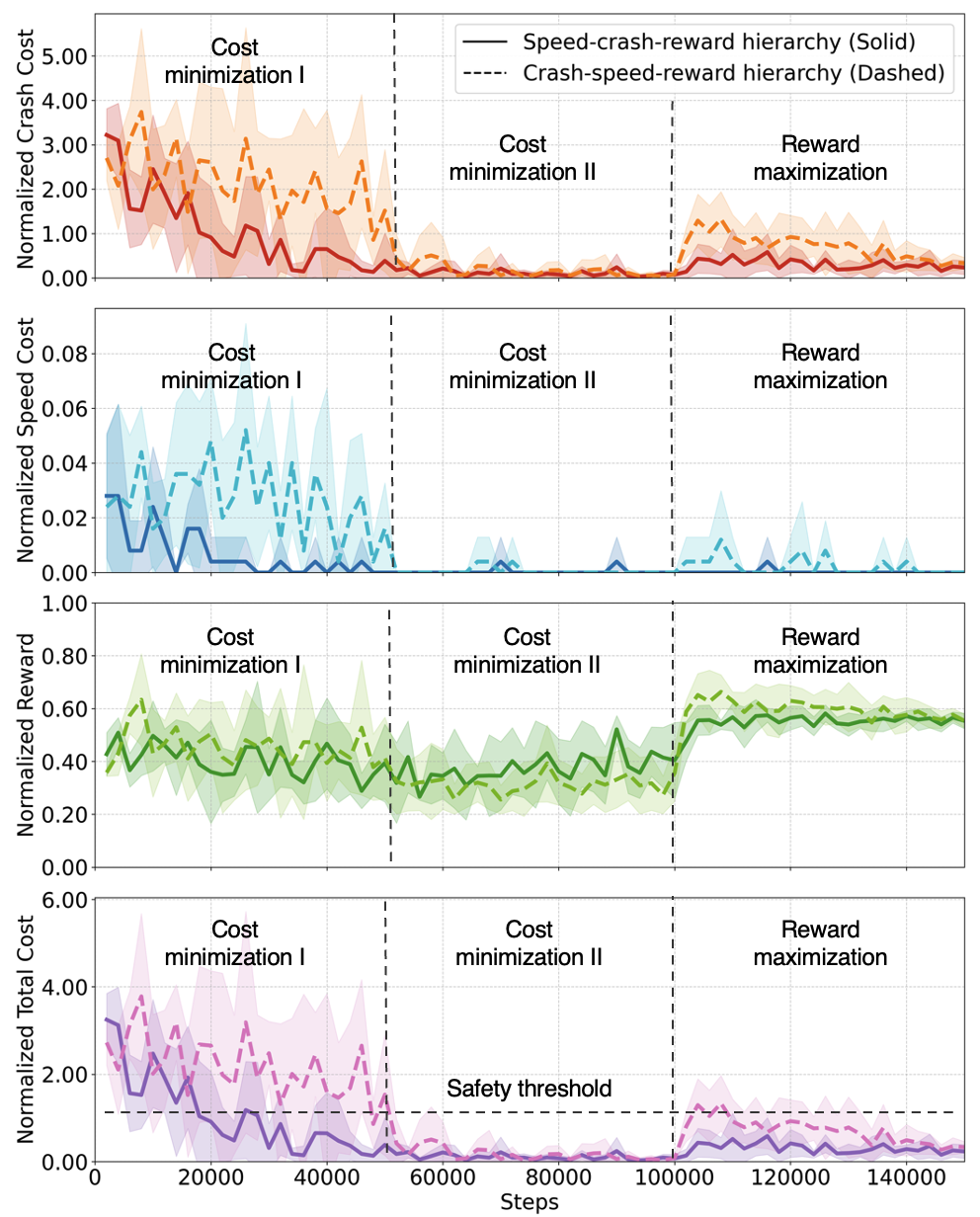}
    \caption{
    Ablation study showing LexiSafe's adherence to sequential lexicographical optimization. For both hierarchy orders, LexiSafe proceeds through the intended phases: first minimizing the primary cost, then the secondary, and finally improving reward while maintaining satisfied constraints.}
    \label{fig:lexi_multi_cost}
\end{figure}

LexiSafe-MC follows the multi-phase training procedure in Algorithm~\ref{alg:lexisafe_multi_single_theta_clean}, sequentially optimizing each objective in order of priority. For comparison, we construct a weighted IQL baseline that replaces the reward with a weighted sum, i.e., $\tilde R = R - \sum_j w_j C_j$, where each $w_j$ encodes the relative importance of cost $j = \{\text{speed}, \text{crash}\}$. We fix $w_{\text{speed}} = 1$ and sweep $w_{\text{crash}} \in \{10, 100, 1000, 5000\}$ to evaluate sensitivity to weight tuning. All experiments use the same offline datasets and seeds, and performance is reported as mean $\pm$ std across seeds.

We first verify whether LexiSafe-MC enforces the specified safety hierarchy under both orderings. Figure~\ref{fig:lexi_multi_cost} shows the mean (dashed and solid lines) and variance (shaded area) of normalized costs and reward across multiple seeds. Regardless of whether crash or velocity is prioritized, LexiSafe-MC consistently satisfies the highest-priority cost before optimizing subsequent objectives. For clarity, we describe the crash-speed-reward hierarchy (dashed line) as a representative example. In Phase 1 (Cost minimization I), the primary crash cost is reduced while the secondary speed cost remains relatively high. Once the crash constraint is met, Phase 2 (Cost minimization II) begins, during which the speed cost is minimized while the crash cost remains below its threshold. Finally, in Phase 3 (Reward maximization), LexiSafe-MC increases reward while maintaining both crash and speed constraints. A similar progression is observed under the speed-crash-reward hierarchy (solid line), where speed regulation dominates early training followed by crash safety enforcement. This sequential pattern confirms that lexicographic policy extraction and phased updates reliably follow the user-specified priority ordering.

We next assess whether simpler weighting strategies can provide comparable guarantees. To this end, we evaluate weighted IQL across a range of crash-weight values. As shown in Figure~\ref{fig:lexi_vs_flat}, weighted IQL fails to produce satisfactory policies across all tested settings. The brackets in the legend of Figure~\ref{fig:lexi_vs_flat} represents ($w_{crash}, w_{speed}$). When $w_{\text{crash}}$ is small (e.g., 1), the policy frequently violates the crash constraint to maximize reward. Increasing $w_{\text{crash}}$ does not reliably reduce violations and still fails to achieve safety levels comparable to LexiSafe-MC throughout the sweep. In contrast, LexiSafe-MC achieves consistent constraint satisfaction without extreme weighting while preserving high reward. This also highlights a practical tuning advantage: cost satisfaction in LexiSafe-MC follows directly from the multi-phase training procedure, whereas weighted IQL requires sensitive per-task weight sweeps that depend heavily on data composition and OOD penalties.

\begin{figure}[t!]
    \centering
    \includegraphics[width=0.9\linewidth]{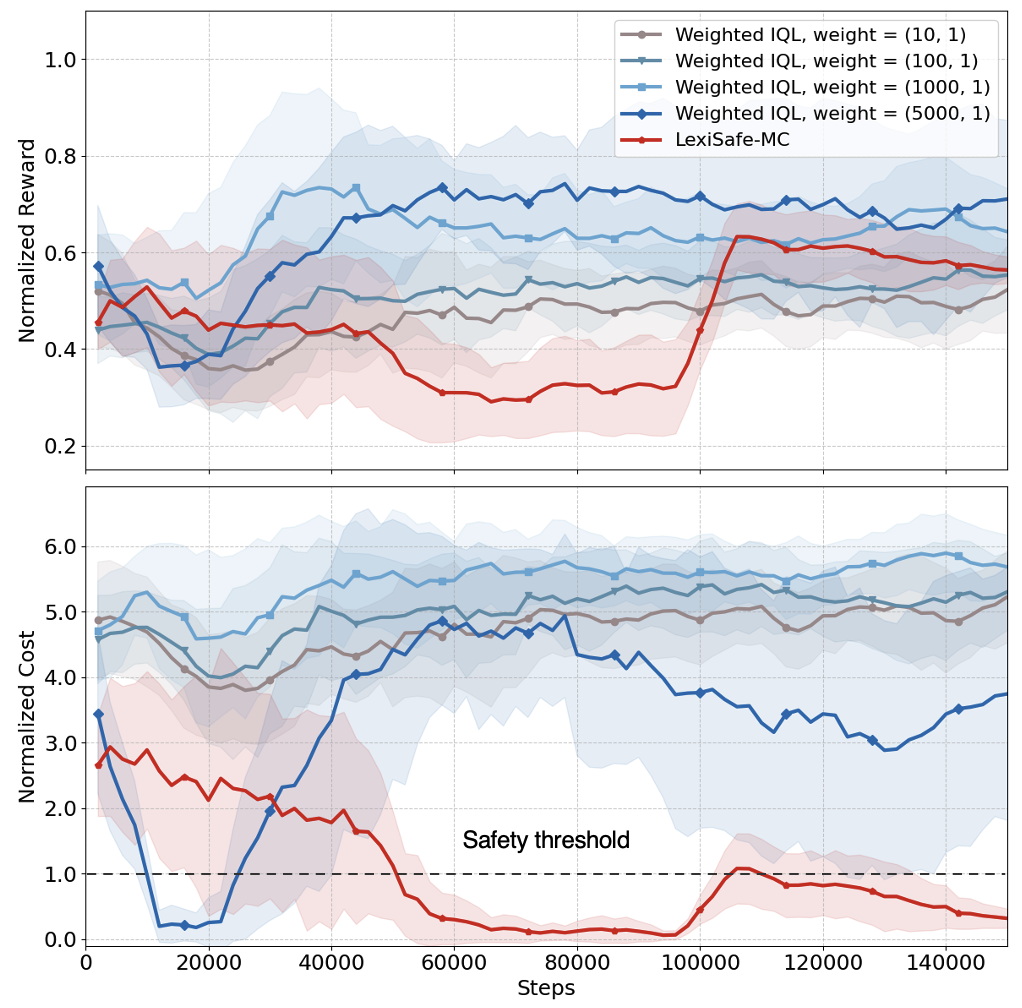}
    \caption{Comparison of LexiSafe-MC and weighted IQL across different crash-weight values with MetaDrive. LexiSafe-MC satisfies safety constraints while maintaining high reward by explicitly enforcing the user-specified priority order through sequential lexicographic optimization. In contrast, weighted IQL struggles to satisfy constraints reliably using the traditional weighted-sum strategy, highlighting both the practical tuning challenges and the limitations of flat weighting approaches. The brackets in the legend represents ($w_{crash}$,$w_{speed}$)}.
    
    \label{fig:lexi_vs_flat}
\end{figure}

\noindent\textbf{Limitations.} Our analysis assumes access to a well-behaved offline dataset with sufficient coverage of safe behaviors, and relies on concentrability coefficients that may be difficult to estimate or validate in practice. In addition, the VC-dimension–based bounds, while providing clear theoretical insight, may be loose for deep nonlinear networks and thus conservative relative to empirical performance. These assumptions highlight common limitations in offline RL theory and point to opportunities for developing tighter and data-dependent characterizations of safety and generalization.

\section{Conclusions}
We propose a principled offline safe reinforcement learning framework termed LexiSafe that enforces a lexicographic order between safety and performance. By solving multi-phase optimization, we preserve safety-critical representations while enabling performance improvement. Our theoretical analysis establishes safety violation and performance suboptimality bounds, highlighting the role of function class complexity and distributional shift (via concentrability). We also construct sample complexity bounds for LexiSafe in the single-cost and multi-cost scenarios. The proposed approaches yield practical and theoretically grounded guarantees for safe policy learning from offline data. This work supports safe RL deployment in high-stakes domains by decoupling safety and performance and providing offline guarantees, addressing the need for reliability in safety-critical applications.

\section*{Acknowledgment}
This work is supported by the COALESCE: COntext Aware LEarning for Sustainable CybEr-Agricultural Systems (NSF $\#$1954556), AI Institute for Resilient Agriculture (USDA-NIFA $\#$2021-647021-35329) and NSF CNS-2313104.

\bibliographystyle{IEEEtran}
\bibliography{references}

\end{document}